\documentclass[openacc]{rsproca_new}%%%%where rsproca is the template name
\usepackage{comment}

\newtheorem{theorem}{\bf Theorem}[section]

\usepackage[english]{babel}
\usepackage[utf8]{inputenc}
\usepackage{balance}
\usepackage{xspace}
\usepackage{bbm}
\usepackage{rotating}
\usepackage{url}
\usepackage{csquotes}
\usepackage{subcaption}
\usepackage{caption}
\usepackage{paralist}
\usepackage{multirow}
\usepackage{multicol}
\usepackage{amsmath,amssymb,mathtools,amsthm}
\usepackage{booktabs}
\usepackage{array}
\usepackage{algorithm}
\usepackage[noend]{algpseudocode}
\usepackage{textcomp}
\usepackage{siunitx}
\usepackage{microtype}
\usepackage{safecolours}
\usepackage{pgfplots}
\pgfplotsset{compat=newest,unit code/.code={\si{#1}},plot coordinates/math parser=false,ylabel right/.style={
        after end axis/.append code={
            \node [rotate=90, anchor=north] at (rel axis cs:1,0.5) {#1};
        }   
    },
}
\usepgfplotslibrary{units,external,groupplots,fillbetween}
\usetikzlibrary{positioning,angles,quotes,patterns,shapes,spy, shapes.misc,backgrounds}
\graphicspath{{./figures/}}

\usepackage{scalerel,stackengine}

\makeatletter
\newcommand{\myitem}[1]{%
\item[#1]\protected@edef\@currentlabel{#1}%
}
\makeatother

\newtheorem{defi}{Definition}

\newtheorem{remark}{Remark}

\newcommand{\fakepar}[1]{\vspace{1mm}\noindent\textbf{#1.}}

\DeclareSIUnit{\belmilliwatt}{Bm}
\DeclareSIUnit{\dBm}{\deci\belmilliwatt}

\DeclareMathOperator*{\E}{\mathbb{E}}

\DeclareMathOperator*{\R}{\mathbb{R}}

\let\originalleft\left
\let\originalright\right
\renewcommand{\left}{\mathopen{}\mathclose\bgroup\originalleft}
\renewcommand{\right}{\aftergroup\egroup\originalright}

\newcommand\figref[1]{Figure~\ref{#1}}

\newcommand\secref[1]{Section~\ref{#1}}

\newcommand{\eg}{e.g.,\xspace}
\newcommand{\ie}{i.e.,\xspace}

\newcommand{\capt}[1]{\mdseries{\emph{#1}}}

\titlehead{Research}

\usepackage{fancyhdr}
\newcommand{\mytitle}{\textbf{Accepted final version.}
To appear in \textit{Philosophical Transactions of the Royal Society A: Mathematical, Physical and Engineering Sciences}.}
\begin{document}

%%%% Article title to be placed here
\title{Ergodicity in\\ reinforcement learning}

\author{%%%% Author details
Dominik Baumann$^{1}$, Erfaun Noorani$^{2}$, Arsenii Mustafin$^{1}$,
Xinyi Sheng$^{1}$, Bert Verbruggen$^{3}$, Arne Vanhoyweghen$^{3}$, Vincent Ginis$^{3,4}$, and Thomas B.\ Sch\"{o}n$^{5}$}

%%%%%%%%% Insert author address here
\address{$^{1}$Cyber-physical Systems Group, Aalto University, Espoo 02150, Finland\\
$^{2}$Control \& Autonomous Systems Engineering Group, MIT Lincoln Laboratory, Lexington, MA 02421-6426, USA\\
$^{3}$Data Analytics Lab, Vrije Universiteit Brussel, Brussel 1050, Belgium\\
$^{4}$School of Engineering and Applied Sciences, Harvard University, Cambridge, Massachusetts 02138, USA\\
$^{5}$Department of Information Technology, Uppsala University, 75105 Uppsala, Sweden\\}

%%%% Subject entries to be placed here %%%%
\subject{artificial intelligence, statistics, systems theory}

%%%% Keyword entries to be placed here %%%%

%%%% Insert corresponding author and its email address}
\corres{Dominik Baumann\\
\email{dominik.baumann@aalto.fi}}

%%%% Abstract text to be placed here %%%%%%%%%%%%
\begin{abstract}
In reinforcement learning, we typically aim to optimize the expected value of the sum of rewards an agent collects over a trajectory.
However, if the process generating these rewards is non-ergodic, the expected value, i.e., the average over infinitely many trajectories with a given policy, is uninformative for the average over a single, but infinitely long trajectory.
Thus, if we care about how the individual agent performs during deployment, the expected value is not a good optimization objective.
In this paper, we discuss the impact of non-ergodic reward processes on reinforcement learning agents through an instructive example, relate the notion of ergodic reward processes to more widely used notions of ergodic Markov chains, and present existing solutions that optimize long-term performance of individual trajectories under non-ergodic reward dynamics.
\end{abstract}
%%%%%%%%%%%%%%%%%%%%%%%%%%%

\begin{fmtext}

\section{Introduction}
Reinforcement learning (RL) has achieved tremendous success in addressing complex tasks~\cite{wurman2022outracing,bellemare2020autonomous}. At its core lies a deceptively simple yet impactful principle: agents learn

\end{fmtext}

% \rsbreak

%%%%%%%%%% Insert the texts which can accomdate on firstpage in the tag "fmtext" %%%%%

%%%%%%%%%%%%%%% End of first page %%%%%%%%%%%%%%%%%%%%%

\maketitle

\noindent through direct interaction with their environment. In each step of this process, the agent observes the state of the environment, selects an action, and receives a scalar reward that signals the quality of that action. Standard textbooks~\cite{powell2021reinforcement,sutton2018reinforcement,bertsekas2019reinforcement}  commonly define the objective of RL as the maximization of the expected value of the sum of all rewards obtained by an agent. This formulation is also intuitively appealing. If the rewards tell us how ``good'' our actions were, we should choose actions that maximize the rewards we can \emph{expect} to receive.
However, the expected value is the average over infinitely many rollouts of a policy.
This \emph{ensemble average} may fail to capture what a single agent experiences over an infinitely long trajectory. In non-ergodic reward processes, the average across infinitely many rollouts can differ from the time average along a single rollout. In many applications, such as medicine, finance, and robotics, the primary concern is a policy’s sustained performance of a single agent over an extended time horizon.

Consider a delivery robot that earns 100 points for each successful delivery but loses 1 point per time step, down to a minimum of 0. The robot has two choices: a direct route that takes 10 time steps but passes through a crowd, where there is a \SI{1}{\percent} chance on each trip that someone destroys the robot and ends all future deliveries, or a safe route that takes 20 time steps and avoids the crowd entirely. If we only look at the average reward per trip, the direct route seems better; about 89 points compared to 80 for the safe route. However, over many trips, consistently choosing the direct route will almost surely result in the robot’s destruction, leaving it with no future rewards. The safe route, while slower, allows the robot to operate indefinitely and achieve higher returns in the long run. This illustrates how focusing solely on short-term averages can be misleading when considering long-term performance.

The above represents a variation of the Russian Roulette example, which serves as a classical illustration of non-ergodicity~\cite{ornstein1973application}. 
It is closely related to safe reinforcement learning~\cite{brunke2022safe}, where we often aim to maximize a reward while avoiding violations of safety constraints.
Nevertheless, non-ergodicity arises in further contexts and is, thus, more broadly applicable.
For instance, when the state distribution is non-stationary, implying that it evolves over time. 
This is a typical case in continual RL~\cite{elelimy2025rethinking}, and even when the environment is stationary, the agent might perceive it as non-stationary.
This can happen because the environment is too complex for the agent to model, as stated in the ``big world hypothesis''~\cite{javed2024big}, or because there are other agents who are also learning and changing their policies, the typical case in multi-agent RL~\cite{albrecht2024multi}.
Another example for non-ergodic reward dynamics is environments with multiplicative rewards~\cite{baier2025multiplicative}, a typical case in biology or chemistry.
We will discuss those different cases in more detail in \secref{sec:ergodicity}.
There, we will also examine the relation between ergodicity of reward processes and ergodicity of the underlying Markov decision process (MDP), the typical way of modeling RL problems.

%%%%%%%%%%%%%%%%%%%%%%%%%%%%%%%%%%%%%%%%%%%%%
% Added after final submission
%%%%%%%%%%%%%%%%%%%%%%%%%%%%%%%%%%%%%%%%%%%%%
\thispagestyle{fancy}	
\pagestyle{empty}

In this tutorial paper, we investigate the implications of non-ergodic reward processes in RL. Our contributions are as follows:
\begin{itemize}
    \item \textbf{Introducing Non-Ergodic Reward Processes}: We define the concept of non-ergodic reward processes in reinforcement learning in \secref{sec:problem} and discuss why they matter in \secref{sec:ergodicity}. 
    \item \textbf{Illustrative Example}: In \secref{sec:example}, we present a simple yet revealing example that, despite its apparent simplicity, state-of-the-art reinforcement learning algorithms fail to solve. 
    \item \textbf{Broader Perspective on Ergodicity}: We discuss the broader relevance of ergodicity and non-ergodicity in reinforcement learning, connecting the discussion to various application domains in \secref{sec:ergodicity}.
    \item \textbf{Overview of Existing Approaches}: We review three strategies proposed in the literature for handling non-ergodic rewards and explain their core ideas in \secref{sec:solutions}. 
\end{itemize}

\section{Problem setting}
\label{sec:problem}

We formulate the problem in the standard Markov decision process setting, which serves as a common theoretical basis for reinforcement learning.
That is, we consider an environment with states $s_{t_k}\in\mathcal{S}\subseteq\R^n$ in which an agent takes actions $a_{t_k}\in\mathcal{A}\subseteq\R^m$ at a discrete time-step $t_k\in\mathbb{N}$.
In response to those actions, the environment transitions to a new state according to an unknown, potentially non-stationary probability distribution $P$.
Besides, the agent receives a reward $r_{t_k}\in\R$.
Following the standard textbooks, the agent's goal is now to learn a policy $\pi:\,\mathcal{S}\to\mathcal{A}$ that maximizes the expected cumulative reward~$R$ over a fixed time  horizon~$T$
\begin{equation}
    \label{eqn:std_objective}
    {\E}_\pi[R_{t_k}] = {\E}_\pi\left[\sum_{\kappa=0}^Tr_{t_\kappa}\right].
\end{equation}

Once a policy is fixed, an MDP reduces to a \emph{Markov reward process (MRP)}: a Markov chain over states together with a reward function (on states, or on transitions). 
As a first step, we analyze the properties of the reward process in isolation and then use them as a basis to develop the theory of RL ergodicity in Section~\ref{sec:ergodicity}. 

Before developing theory and providing precise definitions, let us give some intuition for the notion of ergodicity. 
Informally, a stochastic process is called \emph{ergodic} when long-run time averages along a single realization coincide with ensemble averages. 
That is, the reward sequence $(r_t)$ an individual agent observes in response to a fixed policy over infinitely many steps matches the average rewards observed by infinitely many agents at a single step. 
In contrast, when the reward process is non-ergodic, this expected return \emph{differs} from the return of the individual agent.
To make the notion of ergodicity precise, let us temporarily introduce the notation $r^{(i)}_{t_k}$ to denote the reward received at time-step~$t_k$ of trajectory~$i$.
Statistically, the trajectories can be interpreted as realizations of a stochastic process.
We now have two basic possibilities to define an average.
We can fix a realization or trajectory~$i$ and average over an infinite time horizon, or we can fix a time-step~$t_k$ and average over infinitely many trajectories.
If these two averages align, the process is said to be ergodic.
\begin{defi}[(Strong) Reward ergodicity]
    \label{def:reward_erg}
    A reward process $r^{(i)}_{t_k}$ is called \emph{ergodic} if, for all time steps~$t_k$ and realizations~$i$, 
    \[
        \underbrace{\lim_{N\to\infty}\frac{1}{N}\sum_{j=0}^Nr^{(j)}_{t_k}}_{=\E[r_{t_k}]} = \lim_{T\to\infty}\frac{1}{T}\sum_{\kappa=0}^Tr^{(i)}_{t_\kappa}
    \]
    almost surely.
\end{defi}
If the reward process $r_{t_k}$ is non-ergodic, the averages differ. 
Thus, optimizing the expected return may not necessarily maximize the long-term performance of the average agent. 
To build intuition, we present a simple illustrative example in the next section.

Note that the formal definition of ergodicity given above requires the equivalence of limits for every time step $t_k$ starting from the initial time step $t_0$. 
This implies that the system starts in a stationary state, which significantly narrows the set of cases we can analyze. 
In practice, however, we often observe systems that may start differently but converge to a steady state in the long run. 
For such systems, we introduce a definition of \emph{asymptotic ergodicity}.

\begin{defi}[Asymptotic reward ergodicity]\label{def:asymp_reward_erg}
    A reward process $r^{(i)}_{t_k}$ is called \emph{asymptotically ergodic} if, for all realizations~$i$, 
    \[
        \lim_{N, t_k \to\infty}\frac{1}{N}\sum_{j=0}^Nr^{(j)}_{t_k} = \lim_{T\to\infty}\frac{1}{T}\sum_{\kappa=0}^Tr^{(i)}_{t_\kappa}
    \]
    almost surely.
\end{defi}

\begin{remark}
\label{rem:discount}
    In RL, we typically add a discount factor $\gamma\in(0,1)$ to~\eqref{eqn:std_objective}.
    That is,~\eqref{eqn:std_objective} changes to ${\E}_\pi\left[\sum_{\kappa=0}^T\gamma^\kappa r_{t_\kappa}\right]$.
    The relationship between discounting, ergodicity, and decision-making is investigated in~\cite{adamou2021microfoundations}.
    How to connect this explicitly to RL is an open research question and beyond the scope of this article.
\end{remark}

% \begin{remark}
%     In \emph{average reward RL}~\cite{zhang2021policy,wei2020model,wei2022provably}, we let $T$ in~\eqref{eqn:std_objective} go to infinity, and divide by the time $T$.
%     This is already close to the time average.
%     However, average RL still relies on an expected value.
%     As we will see in the subsequent section, by removing stochasticity through the expected value, we may focus on policies with exceptionally high reward in vanishingly few cases while failing almost surely.
% \end{remark}

\section{Illustrative example}
\label{sec:example}

We use an adapted version of an example originally introduced in~\cite{peters2016evaluating}. 
This version has been used, for instance, in~\cite{baumann2025reinforcement} in the context of RL.
Suppose an agent starts with an initial return of $R(0)=100$ and at each time step $t_k$ receives a reward $r_{t_k}$ according to the following dynamics:
\begin{align}
    \label{eqn:coin_toss_rew_dynamics}
    r_{t_k} = \begin{cases}
    0.5\alpha R_{t_{k-1}} &\text{ if } \eta = 1,\\
    -0.4\alpha R_{t_{k-1}} &\text{ otherwise},
    \end{cases}
\end{align}
where $\eta$ is a Bernoulli random variable taking values  0 or 1 with equal probability.
The example comes from economics and game theory.
Intuitively, we have an agent with a certain initial wealth (100 in our case) that can invest a fraction of its wealth, indicated in~\eqref{eqn:coin_toss_rew_dynamics} by the parameter $\alpha\in[0,1]$.
We then toss a fair coin.
If it comes up heads, the agent wins \SI{50}{\percent} of its investment; otherwise, it loses \SI{40}{\percent}.
Then, we play again with the new $R_{t_1}$.
As the game is relatively simple, we can write down the solution to~\eqref{eqn:std_objective} in closed form:
\[
    \E[R_T] = R_{t_0} (1+0.05\alpha)^T.
\]
Thus, to maximize the expected return, we should choose $\alpha$ as large as possible.
This matches the intuition: we can win \SI{50}{\percent} but lose only \SI{40}{\percent}, so on average, we should win \SI{5}{\percent} per round.
However, if we simulate the game with $\alpha=1$, we see in \figref{sfig:coin_toss} that of 10 randomly selected agents, all end up with a return close to 0.
This qualitative behavior is consistent, independent of the specific choice of agents: almost all agents fail under the seemingly optimal policy.

\begin{figure}
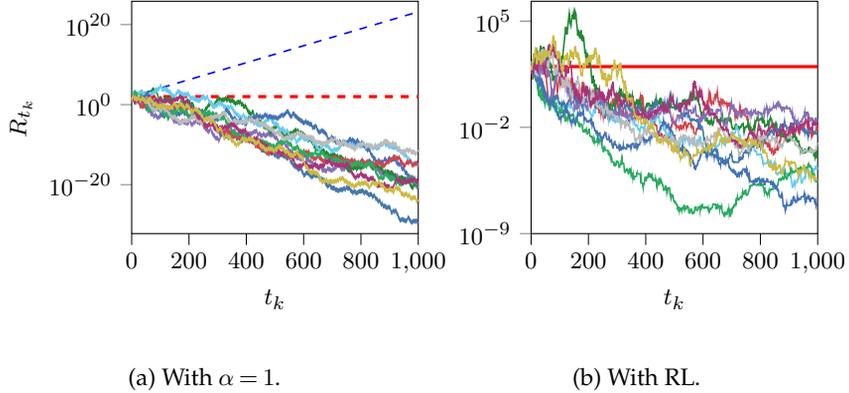

    \centering
    \begin{subfigure}{0.4\textwidth}
        \centering
        \input{tikz/coin_toss}
        \caption{With \(\alpha =1\).}
        \label{sfig:coin_toss}
    \end{subfigure}
    \hspace{0.2cm}
    \begin{subfigure}{0.4\textwidth}
        \centering
        \input{tikz/coin_toss_rl}
        \caption{With RL.}
        \label{sfig:coin_toss_rl}
    \end{subfigure}
    \caption{The coin-toss example. \capt{Both with the policy analytically optimizing the expected return (left) and the RL policy (right), the agents end up with close to 0 return. The red lines mark the initial return, while the dashed blue line in the left plot shows the expected return. The other trajectories represent different realizations of the game. Both figures are taken from~\cite{baumann2025reinforcement}.}}
    \label{fig:coin_toss}
\end{figure}

Why is that?
\figref{fig:coin_toss_explanation} shows all possible realizations of the game for two iterations.
If we average over the possible realizations after one and two iterations, we indeed see that the return grows with \SI{5}{\percent} as expected.
However, we also see that only one out of four realizations results in the agent ending up with a higher return than it started with.
Thus, the probability of ending up on a trajectory with positive growth is only \SI{25}{\percent}, and this probability shrinks exponentially to a set of measure 0 as we let time go to infinity.
In line with this observation,~\cite[Appendix]{hulme2023reply} analyzes the dynamics of the \emph{most likely value} as opposed to the expected value for this game.
From this analysis, we can see that the most likely value shrinks by approximately~\SI{5}{\percent} per round.
That is, while the expected return grows exponentially, the most likely return shrinks exponentially, at approximately the same rate.

\begin{figure}
    \centering
    \begin{tikzpicture}

\tikzstyle{circ} = [draw, circle, minimum size=3em]
\tikzset{>=latex}
\node[circ](start){100};

\node[circ, above right = 1em and 3em of start.east](win1){150};
\node[circ, below right = 1em and 3em of start.east](loss1){60};
\node[below = 6em of start] (overall1){100};
\node at(loss1|-overall1)(overall2){105};

\draw[->, wong-colour4](start) -- (win1);
\draw[->, wong-colour5](start) -- (loss1);
\draw[->] (overall1) -- node[midway, above]{$\cdot1.05$}(overall2);

\node[circ, above right = 1em and 3em of win1.east](win2){225};
\node[circ, below right = 1em and 3em of win1.east](loss21){90};
% \node[circ, above right = 1em and 3em of loss1.east](loss22){90};
\node[circ, below right = 1em and 3em of loss1.east](loss23){36};
\node at(loss23|-overall1)(overall3){110.25};

\draw[->, wong-colour4](win1) -- (win2);
\draw[->, wong-colour5](win1) -- (loss21);
\draw[->, wong-colour5](loss1) -- (loss21);
\draw[->, wong-colour5](loss1) -- (loss23);
\draw[->] (overall2) -- node[midway, above]{$\cdot1.05$}(overall3);
\end{tikzpicture}    
    \caption{Possible realizations of the coin-toss example. \capt{We can see that after two iterations of the game, the agent wins on average, but there is only one out of four possible realizations that lead to such a winning outcome. Taken from~\cite{baumann2025reinforcement}.}}
    \label{fig:coin_toss_explanation}
\end{figure}
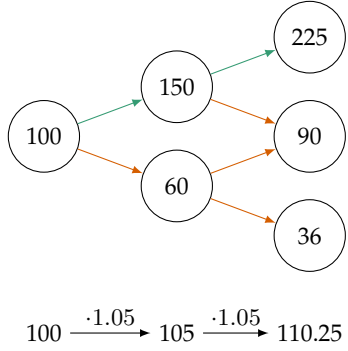

While this explains why the optimal solution following the expected value is not optimal for the average agent, it does not show that an RL agent fails to solve this game.
After all, the game is fairly simple, and RL algorithms have shown their success in much more complex tasks.
Besides, those algorithms do not analytically optimize~\eqref{eqn:std_objective}.
Thus,~\cite{baumann2025reinforcement} has trained proximal policy optimization (PPO)~\cite{schulman2017proximal} to solve the game.
We show the outcome in \figref{sfig:coin_toss_rl}.
Although the policies learned by PPO are not as detrimental as the static $\alpha=1$ policy, they still exhibit negative growth, ultimately resulting in a lower return than they started with.

\section{Ergodicity in reinforcement learning}
\label{sec:ergodicity}

In this section, we relate the notion of ergodic reward processes to that of ergodic Markov chains, where the latter are widely used in RL. 
We then discuss relevant cases of ``ergodicity-breaking,'' \ie settings in which we encounter non-ergodic rewards.

\subsection{Ergodic rewards vs.\ ergodic Markov chains}

We provide a quick overview of the theoretical aspects of ergodicity and show what kinds of Markov chain state dynamics imply reward processes that are ergodic or asymptotically ergodic. 
Throughout, we assume that the transition kernel $P$ is stationary and that the MDP has a finite number of states and actions. Additionally, we assume that the received reward is a deterministic function of the chosen action $a$ and the transition $(s,s')$.
We will provide sufficient conditions; note that necessary conditions are not our focus here---for example, if all rewards in the MDP are constant, then all reward processes produced by it are trivially ergodic.

We start by giving formal definitions of an ergodic Markov chain and an ergodic MRP.

\begin{defi}[Ergodic Markov chain]
\label{def:ergodic_chain}
A Markov chain is \emph{ergodic} if it is both irreducible and aperiodic, meaning it is possible to reach any state from any other state, and the chain does not get trapped in a cycle.
\end{defi}

\begin{defi}[Ergodic Markov reward process]
\label{def:ergodic_mrp}
A Markov reward process is \emph{ergodic} if its underlying Markov chain is ergodic according to Definition~\ref{def:ergodic_chain}.
\end{defi}

However, even an ergodic MRP is not sufficient to guarantee an ergodic reward process. 
To satisfy the “every time step” requirement, we also need to control the initial state distribution, which leads to the following theorem.

\begin{theorem}
    If an MRP is ergodic following Definition~\ref{def:ergodic_mrp} and the initial state is sampled from the stationary distribution, the observed reward process is ergodic as defined in Definition~\ref{def:reward_erg}.
\end{theorem}

\begin{proof}
Fix a finite-state Markov reward process with transition matrix $P$ and a bounded transition-based reward $g:\mathcal{S}\times\mathcal{S}\to\mathbb{R}$. Let $\pi$ be a stationary distribution ($\pi^\top P=\pi^\top$), and sample the initial state $X_0\sim\pi$. Then $(X_t)$ is stationary, hence so is the reward sequence $R_t\coloneqq g(X_t,X_{t+1})$:
\[
\E[R_{t_k}]
=\sum_{s,s'} \pi(s) P(s,s')\, g(s,s') \eqqcolon \rho
\quad \text{for all } t_k.
\]
Therefore, for every $t_k$,
\[
\underbrace{\lim_{N\to\infty}\frac{1}{N}\sum_{j=0}^N r^{(j)}_{t_k}}_{=\E[R_{t_k}]=\rho}
\;=\;\rho.
\]
Moreover, by the strong law for additive functionals of (stationary) finite Markov chains,
\[
\lim_{T\to\infty}\frac{1}{T}\sum_{\kappa=0}^T r^{(i)}_{t_\kappa}
=\lim_{T\to\infty}\frac{1}{T}\sum_{\kappa=0}^T g(X_{t_\kappa},X_{\kappa+1})
=\rho
\quad\text{almost surely.}
\]
Combining the two yields, for all $t_k$ and $i$,
\[
\underbrace{\lim_{N\to\infty}\frac{1}{N}\sum_{j=0}^N r^{(j)}_{t_k}}_{=\E[R_{t_k}]}
=\lim_{T\to\infty}\frac{1}{T}\sum_{\kappa=0}^T r^{(i)}_{t_\kappa}
\quad\text{a.s.}
\]
\end{proof}

As discussed in Section~\ref{sec:problem}, a strongly ergodic reward process is hard to achieve; we therefore relax the assumption on the initial state distribution and aim for an \emph{asymptotically} ergodic process. We also relax irreducibility: we require that, regardless of the initial state, the chain converges to a single stationary distribution of states. Concretely, we assume the underlying Markov chain is aperiodic and has a single recurrent (connectivity) class—such a chain is called a \emph{unichain}. In this case, the stationary distribution is unique and the chain converges to it geometrically~\cite[Thm.~16.0.1]{MeynTweedie2009}.

\begin{remark}
The aperiodicity assumption is necessary but not restrictive in practice: periodicity is typically broken by the presence of a single action with nonzero self-loop probability (i.e., a positive probability of remaining in the same state).
\end{remark}

\begin{theorem}
If an MRP's underlying Markov chain is unichain and aperiodic, then the resulting (transition-based) reward process is asymptotically ergodic as defined in Defintion~\ref{def:asymp_reward_erg}.
\end{theorem}

\begin{proof}
    Let $(X_t)$ be a Markov chain on a finite state space with transition matrix $P$, and let $g:\mathcal{S}\times\mathcal{S}\to\mathbb{R}$ be bounded. Under the unichain and aperiodic assumptions, there is a unique stationary distribution $\pi$ supported on the unique recurrent class, and for any initial distribution $\mu_0$ we have geometric convergence to stationarity~\cite[Thm.~16.0.1]{MeynTweedie2009}:
\[
\mu_0 P^{t_k}\;\xrightarrow[k\to\infty]{}\;\pi^\top
\quad\text{geometrically in total variation.}
\]
Hence, for the ensemble at time $t_k$,
\[
\lim_{t_k\to\infty}\;\E[R_{t_k}]
=\lim_{t_k\to\infty}\sum_{s,s'} \bigl(\mu_0 P^{t_k}\bigr)(s)\,P(s,s')\,g(s,s')
=\sum_{s,s'} \pi(s)P(s,s')\,g(s,s')
\eqqcolon \rho.
\]
Therefore,
\[
\lim_{N,t_k\to\infty}\frac{1}{N}\sum_{j=0}^N r^{(j)}_{t_k}
=\rho.
\]
On the other hand, by the ergodic theorem for additive functionals on the unique recurrent class,
\[
\lim_{T\to\infty}\frac{1}{T}\sum_{\kappa=0}^T r^{(i)}_{t_\kappa}
=\lim_{T\to\infty}\frac{1}{T}\sum_{\kappa=0}^T g(X_{t_\kappa},X_{\kappa+1})
=\rho
\quad\text{almost surely, independently of }X_0.
\]
Combining the two limits gives
\[
\lim_{N,t_k\to\infty}\frac{1}{N}\sum_{j=0}^N r^{(j)}_{t_k}
=\lim_{T\to\infty}\frac{1}{T}\sum_{\kappa=0}^T r^{(i)}_{t_\kappa}
\quad\text{a.s.}
\]
\end{proof}

How do these notions extend to an MDP when multiple action choices are available at each state? There are several approaches. One classical option is the following definition of an ergodic MDP:

\begin{defi}[Ergodic MDP~{\cite[Ch.~8.3]{puterman2014markov}}]\label{def:ergodic_MDP}
    An MDP is called \emph{ergodic} if the transition matrix corresponding to every deterministic stationary policy consists of a single recurrent class.
\end{defi}

While this definition yields strong guarantees, the condition is quite restrictive. Recent research shows that good results can be obtained for a broader class of MDPs. In particular, we require a \emph{unichain} structure and a \emph{unique optimal policy}, which implies an asymptotically ergodic reward process. For known MDPs, this assumption yields geometric convergence to the (approximately) optimal policy under value iteration~\cite{mustafin2025geometryavrew}. Thus, we typically do not need every single policy to be unichain; some policies may have multiple disconnected classes, some of which are “bad” in terms of received rewards. For ergodicity of the optimal policy and fast convergence to it, it is sufficient that all states are \emph{escapable}, \textit{i.e.}, the agent can leave them and reach the optimal recurrent class.

\subsection{Ergodicity-breaking}

Based on these results, we can now define more precisely for which types of environments ergodicity ``breaks.''
Starting with the basic definition of the Markov reward process, we require that rewards are defined on the states of the MDP or on the state transitions.
In the coin-toss example, the reward at a given time step depends on the rewards collected up to that time step and, by extension, on the history leading up to it.
This violates the Markov assumption.
Thus, we generally deal with non-ergodic reward dynamics when the reward at a given time step depends on rewards collected in prior time steps.
A special case of this is MDPs with multiplicative rewards, such as the coin-toss game.
MDPs with multiplicative rewards are more widely discussed in~\cite{baier2025multiplicative}.
Multiplicative rewards are common in the real world; many biological and chemical processes exhibit multiplicative growth.

A simple way around having multiplicative rewards in the coin-toss game is to introduce wealth as a state.
Then, the reward at a given time step depends only on the current state and the outcome of the coin toss.
Nevertheless, in this case, the state distribution is non-stationary: the wealth fluctuates and the variance diverges as time goes to infinity~\cite[Appendix]{hulme2023reply}.
We can generalize this case.
Consider a regulation problem, \eg we want to learn a control policy that regulates the water level of a tank around an equilibrium.
If we learn a proper policy, the state distribution will converge to a stationary distribution around the desired water level.
Instead, if we want a robot to cover as much distance as possible, the distance should, by design, not converge to a stationary distribution.

Another possibility for losing stationarity in the state distribution is if the probability distribution modeling the state transitions itself is non-stationary.
This is a case often considered in continual RL~\cite{elelimy2025rethinking}, where we explicitly aim to enable agents to adapt to changing environmental dynamics.
Another case is transfer learning~\cite{zhu2023transfer}, where we similarly aim to adapt to environmental changes.
Nevertheless, even if the environment itself is stationary, it might appear non-stationary to the learning agent, as stated in the ``big world hypothesis''~\cite{javed2024big}.
An example for this is distributed multi-agent RL~\cite{albrecht2024multi}.
In distributed multi-agent RL, all agents are learning and adapting their policies, but each individual agent has no information about the learning process of other agents.
Thus, following the same policy may yield diverging rewards as the other agents adapt their policies.
Therefore, from the perspective of the individual agent, the environment appears non-stationary.

Our results from the previous subsection further require a unichain Markov chain.
Thus, another example of ergodicity-breaking is multi-chain MDPs with disconnected sub-MDPs~\cite{sun2009rollout,atia2020steady}.
Consider the delivery robot from the introductory example.
Suppose the robot could be deployed in either a small town or a big city.
Clearly, the probability of being destroyed and, hence, the optimal policy differ in both cases.
Optimizing an expected reward that averages between the two possible initializations may favor a policy that is suboptimal in both potential realizations.
Thus, the robot should learn a policy that optimizes the long-term performance given the specific realization.

Finally, we assumed the underlying Markov chain to be irreducible, \ie any state can be reached from any other.
This property is violated if we have an absorbing state.
Absorbing states could, for instance, represent safety constraints from which we cannot recover.
Such as the robot in the introductory example, which cannot recover from being destroyed.
Such ``fatal states'' are often considered in safe RL~\cite{brunke2022safe}.
Thus, also in safe RL, we often deal with non-ergodic reward processes.

\subsection{Related work}
\label{sec:rel_work}

In the following sections, we will discuss three approaches that explicitly address non-ergodic reward structures in RL.
Nevertheless, other works have also considered similar settings.
For instance, by considering non-ergodic MDPs or non-Markovian reward structures, or by diverging from the expected return as an optimization criterion to improve robustness.

\fakepar{Non-ergodic MDPs}
When ergodicity is considered in RL, it is typically about ergodicity of the MDP or the underlying Markov chain.
For instance,~\cite{pesquerel2022imed,ok2018exploration,agarwal2022regret} derive guarantees for RL by explicitly assuming the MDP to be ergodic.
Other works assume the existence of an ``absorbing barrier,'' \ie a state from which the agent cannot recover.
They then provide guarantees for avoiding such states and explore only an ergodic sub-MDP~\cite{turchetta2016safe,heim2020learnable}.
The case of non-Markovian rewards has been studied as well.
When rewards are non-Markovian, they are not only dependent on the current state and action.
Nevertheless~\cite{majeed2018q} still provides convergence results for Q-learning under such a reward structure, and \cite{schmidhuber1990reinforcement,gaon2020reinforcement} propose dedicated algorithms that can handle non-Markovian rewards.
None of those works discuss the connection between non-ergodic MDPs or non-Markovian rewards with the divergence of time and ensemble statistics of the return.

\fakepar{Reward transformations}
One solution to the problem of non-ergodic rewards is to transform returns and optimize the increments of the transformed returns.
Such transformations have also been proposed in risk-sensitive RL~\cite{mihatsch2002risk,shen2014risk,fei2021risk,noorani2021reinforce,noorani2022Exprisk,prashanth2022risk} and reward-weighted regression~\cite{peters2007reinforcement,peters2008learning,wierstra2008episodic,abdolmaleki2018relative,peng2019advantage}. 
In those works, the transformation is often supported by more intuitive arguments.
However, there is a concrete connection between those transformations and the one we discuss in the next section.
In particular, for specific reward dynamics, both are equivalent.
This connection has been analyzed in more detail in~\cite{baumann2025reinforcement}.

\fakepar{Average-reward RL}
In \secref{sec:example}, we have shown that for a non-ergodic process, the expected return and the return of a single, infinite trajectory diverge.
If we are interested in the infinite trajectory, one may also consider average-reward RL as an optimization criterion.
In average-reward RL, we typically let $T$ in \eqref{eqn:std_objective} go to infinity and divide by $T$.
This is almost exactly the definition of the time average from Definition~\ref{def:reward_erg} (right-hand side of the equation).
The difference is that we still take an expectation over the rewards.
Average-reward RL has originated in dynamic programming~\cite{howard1960dynamic,blackwell1962discrete,veinott1966finding} and been studied since the early days of RL~\cite{mahadevan1996average}.
It remains an active field to this day~\cite{zhang2021policy,wei2020model,wei2022provably}.
While the general objective, optimizing along time instead of across realizations, seems similar, average-reward RL still takes an expected value.
By first taking an expected value in the coin-toss example, we remove stochasticity, and choosing $\alpha=1$ would be the optimal choice.
Thus, average-reward RL would not solve the ergodicity problem.

\section{Solutions}
\label{sec:solutions}

Recent work proposes various strategies for RL with non-ergodic reward processes.
In this section, we briefly introduce those approaches.

\subsection{Learning ergodicity transformations}
\label{sec:dominik}

In~\cite{peters2018time}, the authors show how---through so-called ``ergodicity transformations''---an ergodic observable can be extracted from a non-ergodic stochastic process.
Optimizing the expected value of this ergodic observable then corresponds to optimizing the \emph{time-average growth rate} of the process.
The time-average growth rate is the growth rate of the individual agent averaged over an infinitely long trajectory.
However, their approach requires an analytical model of the stochastic process.
In RL, we typically assume that the environment dynamics and, hence, also the reward dynamics are unknown.
Thus,~\cite{baumann2025reinforcement} proposes an algorithm for learning an ergodicity transformation from data.
The approach is inspired by variance-stabilizing transformations~\cite{bartlett1947use} and consists of two steps 
\begin{enumerate}
    \item Use LOESS (locally estimated scatter-plot smoothing)~\cite{cleveland1979robust} to plot $R_{t_k}$ against $\log(r_{t_k}^2)$;
    \item Interpolate a function $h$, which is then the desired transformation.
\end{enumerate}
The RL agent is then trained on the increments of transformed returns, $\Delta h(R_{t_k}) \coloneqq h(R_{t_k}) - h(R_{t_{k-1}})$.
By applying this algorithm to the coin-toss game, again with PPO as an underlying RL algorithm, the agent can learn a winning policy as shown in \figref{fig:coin_toss_learned_trans}.
In~\cite{baumann2025reinforcement}, the authors further show that the algorithm generalizes to more complex environments, showing increased performance on the popular cart-pole and reacher environments.

\begin{figure}
    \centering
    \input{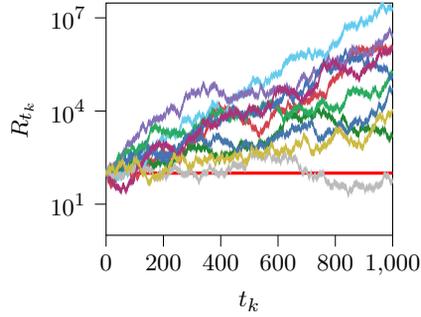}
    \caption{Coin-toss with learned transformation. \capt{By learning on increments of transformed returns, we can learn a winning policy. The red line marks the initial return; the other colored plots represent different realizations of the game. Taken from~\cite{baumann2025reinforcement}.}}
    \label{fig:coin_toss_learned_trans}
\end{figure}

A key limitation of this approach is its reliance on access to a trajectory of returns, which is necessary for learning an appropriate transformation. Therefore, it is currently restricted to settings like the coin-toss, where a single run with $\alpha=1$ already reveals the dynamics, or Monte Carlo-type algorithms, such as REINFORCE~\cite{williams1992simple}, that first collect a return trajectory before updating the policy.

\subsection{Modified geometric mean estimator}
\label{sec:xinyi}

Instead of explicitly estimating the ergodicity transformation,~\cite{sheng2025beyond} formulates the objective function as a convex combination of the traditional RL objective and the time-average growth rate $ G^\pi_\infty $ of an infinitely long trajectory under policy $\pi$ controlled by a tuning parameter $\lambda\in[0,1]$:
\begin{equation}
    \max_{\pi} \left\{ (1-\lambda) \mathbb{E}^\pi\left[ \sum_{\kappa=0}^\infty \gamma^\kappa r_{t_\kappa}  \right] + \lambda G^\pi_\infty  \right\}.  \label{equ:reg_objective}
\end{equation}
Here, the time-average growth rate, which is constant for a fixed policy $\pi$, acts as a regularizer in a regularized RL framework. 
Assuming we can directly observe the time-average growth rate, the authors derive a corresponding Bellman operator based on this formulation. 
% To enable practical implementation, the authors apply a systematic scheme of regularized modified policy iteration~\cite{geist2019theory}, consisting of a greedy step and an evaluation step.

In practice, we need to estimate the time-average growth rate.
Crucially, in non-ergodic reward dynamics, we must estimate it from individual trajectories rather than population-level statistics. Learning from a single trajectory is essential for capturing the path dependencies. 
The geometric mean is proven to be a valid estimator of the time-average growth rate under multiplicative dynamics. 
To mitigate the limitation of finite samples and avoid tracking the entire history,~\cite{sheng2025beyond} develops a modified version based on an $N$-sliding window, which allows long-term reward features to be captured along a single trajectory. 
Moreover, since the single-step ($N=1$) size is insufficient to capture dynamic evolution, multi-step Q-learning~\cite{peng1994incremental} is employed, with the same window size $N$ used for both estimation and value propagation. 

Applying this algorithm with $\lambda=1$ to the coin-toss game, we can see in \figref {fig:coin_toss_beyond_gm} that this algorithm can also learn a winning strategy.
In~\cite{sheng2025beyond}, the authors further evaluate the algorithm across various $\lambda$ in different benchmark environments, including cart-pole and lunar lander. 
The results demonstrate superior performance compared to standard multi-step Q-learning. 

% Currently, this approach is developed only for discrete action spaces and cannot yet be applied to more complicated simulations. 

\begin{figure}
    \centering
    \scalebox{0.8}{\input{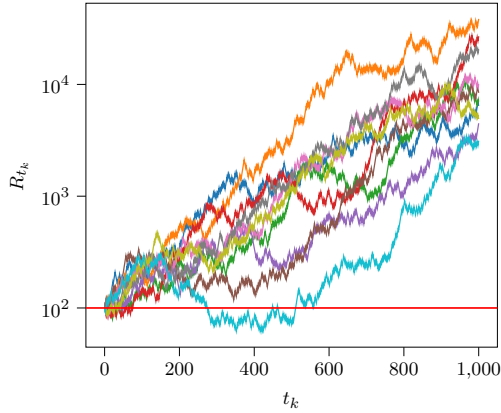}}
    \caption{Coin-toss with geometric mean estimation. \capt{By embedding the geometric mean estimator into the RL objective, we can learn a winning policy. The red line marks the initial return; the other colored plots represent different realizations.}}
    \label{fig:coin_toss_beyond_gm}
\end{figure}

\subsection{Temporal training and path-dependent updates}
\label{sec:bert}
An altogether different approach for environments with multiplicative dynamics is proposed in~\cite{verbruggen2025reinforcement}. 
By explicitly including path dependence in the problem, the agent can learn the temporal dynamics without changing the value function or the rewards it receives. 
This explicit implementation of the temporal dynamic requires the agent to face the same action-selection problem multiple times within a single training episode to optimise for a trajectory and its long-term consequences. 
We can illustrate this process in another simple coin-toss experiment. 
Agents are presented with only two actions and one state. 
In this multi-armed bandit setting, the agent needs to learn the trade-off between taking a safe action with a smaller reward ($r_{\text{safe}}$) and a risky action with a stochastic binomial outcome. 
These rewards are multiplicative on the agent's initial return ($R_{t_0}$) and depend on a predefined, but unknown to the agent, probability $p$ of receiving the worst outcome from the risky action. 
An agent's final return after one step results in
\begin{align}
    \label{eqn:coin_toss_temp_train}
    R_{t_{k+1}} = \begin{cases}
    r_{\text{safe}} \cdot R_{t_k} &\text{ if } a_{t} = \text{safe},\\
    r_{\text{risk,loss}} \cdot R_{t_k} &\text{ if } a_{t} = \text{risk} \text{ with probability } p, \\
    r_{\text{risk,win}} \cdot R_{t_k} &\text{ if } a_{t} = \text{risk} \text{ with probability } (1-p), 
    \end{cases}
\end{align}
where rewards are set such that; $r_{\text{risk, loss}}<1 \leq r_{\text{safe}} < r_{\text{risk, win}}$. 
The parameter $p$ dictates the probability of receiving the worst outcome from the agent taking the risky action. 
The preference of the agent, measured as a probability for choosing the safe action, depends on this parameter~$p$. 
Starting from zero preference for the safe action, this increases up to a certain value for the probability of the worst outcome, where the safe action becomes the most preferred. 
The specific value where the preference exceeds \SI{50}{\percent} is called the indifference point and indicates where the agent becomes risk-averse. 
This indifference point can be calculated theoretically to reflect the policy. 
Depending on the probability of the worst outcome under the risk action, we discern two distinct optimisation strategies. 
These strategies are reflected in the two resulting policies illustrated in \figref{fig:coin-toss-temp-training}. 
The optimal policy, as discussed in \secref{sec:example}, requires the agent to optimise using an alternative to the expected value, the time-dependent growth rate in a multiplicative dynamic. 
These two optimisers, expected value and growth rate, correspond to two different indifference points. 
The suboptimal policy based on expected value optimisation results in a predicted indifference point $p_\mathrm{E}$, whereas the optimal policy based on time growth changes preference at $p_\mathrm{T}$.
%Either the agent uses a traditional optimisation based on expected value,or it uses a path-dependent optimal solution based on time growth. An agent's policy reflects this strategy by studying the changing preference for taking the safe action over the risk action for an increasing probability of receiving the worst outcome on the risk bet, as illustrated in}
\begin{figure}
    \centering
    \includegraphics[width=0.5\linewidth]{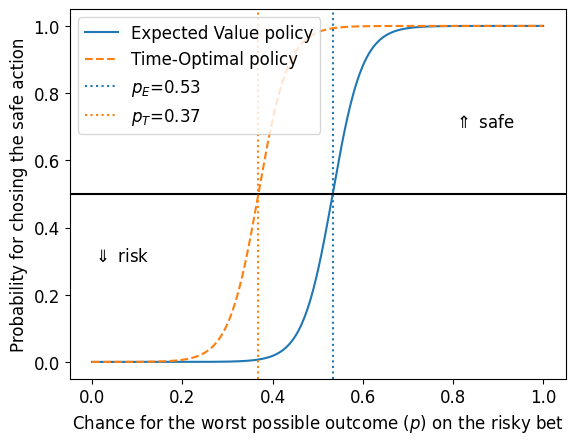}
    \caption{Illustration of two policies for an agent in a coin toss experiment with two actions. \capt{Two different policies indicate the preference for an agent to take a safe action over a risky one, illustrated by the indifference points. An optimal policy changes the action preference based on the prediction of time growth ($p_\mathrm{T}$) rather than the expected value ($p_\mathrm{E}$).}}
    \label{fig:coin-toss-temp-training}
\end{figure}
This toy problem illustrates the trade-off faced by the agent in a non-ergodic environment. 
Traditional RL implementations with various update rules achieve perfect expected value optimisation, leading to a suboptimal policy. 
The solution proposed in~\cite{verbruggen2025reinforcement} extends the agent's horizon to repeat this problem over multiple time steps, each time updating the final return with the new reward and returning to the original state. 
The agent now learns the optimal policy, reflected in a shift of the indifference point from the prediction based on the expected value~$p_\mathrm{E}$ to the optimal policy, where the agent changes preference around~$p_\mathrm{T}$. 
The inclusion of the trajectory encodes the temporal dynamics, and an alternative optimiser based on the growth rate is used in a multi-armed bandit problem.

An alternative formulation of the problem as an MDP introduces time-dependent states ($s_{t_k}$), one for each time step. 
An agent can again learn the temporal dynamics in this setting under the new training procedure, but it takes longer to transition to the optimal policy and requires updates that consider the entire trajectory under a Monte Carlo update rule. 
In such an implementation, the typical value provides the best estimate for the process~\cite{redner1990random}. 
As time increases, the exponential growth in trajectories makes outcomes with excessive return growth exceptionally rare. 
The expected value along each trajectory is now more similar to the most likely value~\cite{hulme2023reply} and, for sufficiently long time steps, aligns with the growth rate.

As in the previous section, we apply temporal training to the coin-toss experiment using deep reinforcement-learning agents~\cite{Verbruggen2026ModelAgnostic}. Figure~\ref{fig:coin_toss_temporal} shows that temporal training systematically improves performance relative to a standard, single-step training setup.

After training, we evaluate the learned policy in two distinct ways. In the first, the policy is fixed and repeatedly applied without further reference to the wealth dynamics used during training: the predicted investment fraction is iterated directly, producing the successive trajectories shown in Fig.~\ref{fig:temporal_result_fixed}. In the second, the same learned policy is applied recursively to the output of the previous step, with normalized current wealth re-introduced as the input at each iteration, shown in Fig.~\ref{fig:temporal_result_var}.

Because temporal training is performed over a fixed time horizon and wealth scale, these two evaluation modes reveal how the learned policy behaves both when decoupled from and when embedded within the underlying wealth dynamics. In both cases, temporal training allows agents to outperform the standard solution to the example experiment presented in Section~\ref{sec:example}.

%Because the model requires hyperparameter tuning, results can vary with a naive implementation. Nonetheless, the agent's resulting policy, reflected by the estimated fraction $\hat{f}=0.27$, closely resembles the Kelly objective $f^{*}=0.25$.
\begin{comment}
\begin{figure}
    \centering
    \input{tikz/coin_toss_temporal_fixed2}
    \caption{Coin-toss with agents using temporal training. \capt{An actor-critic model iterates multiple steps in one training episode for training the agent's policy. The red line marks the initial wealth, and the coloured lines indicate different iterations of the experiment using the learned policy. Training techniques are used as discussed in~\cite{Verbruggen2026ModelAgnostic}. The temporal training method  for an agent results in an optimal strategy compared to standard training, Figure~\ref{fig:coin_toss} as presented in Section~\ref{sec:example}.}}
    \label{fig:coin_toss_temporal}
\end{figure}
\end{comment}
\begin{figure}
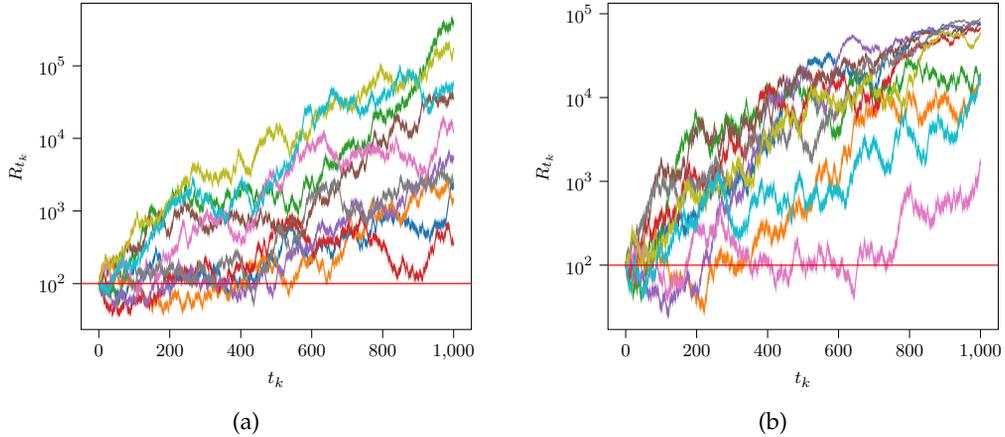

\centering
  %------------- first sub-figure (top) -----------------
  \begin{subfigure}{0.48\textwidth}
    \centering
    \resizebox{\linewidth}{!}{%
    \input{tikz/coin_toss_temporal_fixed}
    }
    \caption{}
    \label{fig:temporal_result_fixed}
  \end{subfigure}
  \hfill
  %------------- second sub-figure (bottom) -------------
  \begin{subfigure}{0.48\textwidth}
    %\centering
    \resizebox{\linewidth}{!}{%
    \input{tikz/coin_toss_temporal_var}
    }
    \caption{}
    \label{fig:temporal_result_var}
  \end{subfigure}
  \caption{Coin-toss with agents using temporal training. \capt{An actor-critic model iterates multiple steps in one training episode for training the agent's policy given an initial wealth indicated by the red line. The colored trajectories in~\ref{fig:temporal_result_fixed} represent successive iterations obtained by repeatedly applying the final policy’s predicted investment fraction. When the trained model is recursively applied to its own outputs, taking normalized wealth as the state input, the resulting iterative dynamics are shown in Fig~\ref{fig:temporal_result_var}. Training techniques are used as discussed in~\cite{Verbruggen2026ModelAgnostic}.}}
  \label{fig:coin_toss_temporal}
\end{figure}

%\end{comment}
\section{Conclusions and open challenges}
In this paper, we discussed the challenge posed by non-ergodic reward processes in RL.
We related the notion of non-ergodic reward processes to more widely used ergodicity notions in the RL community and presented three existing solutions for optimizing long-term performance under non-ergodic reward dynamics.

These three approaches are starting points for resolving the non-ergodicity challenge in RL, but not the final answer.
All three algorithms have, to date, been applied only to relatively simple RL environments, such as the reacher or the lunar lander~\cite{brockman2016openai}.
Extending them to more complex environments requires significant work.
For instance, the algorithm from \secref{sec:dominik} learns a transformation only of the reward.
In the coin-toss game, this is fine, as the structural properties of the game do not change when changing $\alpha$.
However, in more complex environments, one may easily imagine that the transformation should vary with the agent's state and actions.
Moreover, the algorithm currently separates the learning of a transformation from that of an optimal policy.
Joint learning further complicates the problem.
The algorithm from \secref{sec:xinyi} similarly only considers the returns and is currently restricted to environments with discrete action spaces.
Furthermore, it introduces two novel hyperparameters: $\gamma$ to trade off the time-average growth rate and expected value, and $N$ for the step size in multi-step Q-learning and the length of the $N$-sliding window, both of which must be tuned.
The algorithm from \secref{sec:bert} requires learning the temporal dynamics of the states, which is challenging in complex environments.
Further open questions that remain unanswered are the development of a principled empirical measure of non-ergodicity, which would allow us to gauge ``how non-ergodic'' typical RL benchmarks are and how well developed algorithms are at mitigating the problem, and a proper treatment of the discount factor from an ergodicity perspective~\cite{adamou2021microfoundations} (see also Remark~\ref{rem:discount}).

\enlargethispage{20pt}

\ack{This research was partially supported by the Finnish Ministry of Education and Culture through the Intelligent Work Machines Doctoral Education Pilot Program (IWM VN/3137/2024-OKM-4), the Research Council of Finland Flagship programme: Finnish Center for Artificial Intelligence FCAI, and \emph{Kjell och M{\"a}rta Beijer Foundation}.}

%%%%%%%%%% Insert bibliography here %%%%%%%%%%%%%%

\vskip2pc
\bibliographystyle{RS} %%%% .BST file

\bibliography{sample} %%%%% .Bib file

\end{document}